\documentclass{article}

\PassOptionsToPackage{numbers}{natbib}
\usepackage[preprint]{neurips_2026}

\usepackage[utf8]{inputenc} % allow utf-8 input
\usepackage[T1]{fontenc}    % use 8-bit T1 fonts
\usepackage{hyperref}       % hyperlinks
\usepackage{url}            % simple URL typesetting
\usepackage{booktabs}       % professional-quality tables
\usepackage{amsfonts}       % blackboard math symbols
\usepackage{nicefrac}       % compact symbols for 1/2, etc.
\usepackage{microtype}      % microtypography
\usepackage{xcolor}         % colors

\usepackage{wrapfig}
\usepackage{amsmath}
\usepackage{amssymb}
\usepackage{mathtools}
\usepackage{amsthm}
\usepackage{colortbl}
\usepackage{longtable}
\usepackage{mdframed}
\usepackage{color}
\usepackage{xspace}
\usepackage{adjustbox}
\usepackage{enumitem}
\usepackage{setspace}
\usepackage{tikz}
\usepackage{colortbl}
\usepackage{tcolorbox}

\RequirePackage{xspace}
\makeatletter
\DeclareRobustCommand\onedot{\futurelet\@let@token\@onedot}
\def\@onedot{\ifx\@let@token.\else.\null\fi\xspace}

\usepackage[capitalize,noabbrev]{cleveref}

\theoremstyle{plain}
\newtheorem{theorem}{Theorem}[section]

\newtheorem{lemma}[theorem]{Lemma}
\newtheorem{corollary}[theorem]{Corollary}
\theoremstyle{definition}
\newtheorem{definition}[theorem]{Definition}
\newtheorem{assumption}[theorem]{Assumption}
\theoremstyle{remark}

\usepackage[textsize=tiny]{todonotes}

\newcommand{\ours}{\texttt{CPO}\xspace}

\title{Beyond Entropy: Correctness-Aware Advantage Shaping via Contrastive Policy Optimization}

\author{
Weiwen Xu\raisebox{4pt}{\small $1$} \enspace Jia Liu\raisebox{4pt}{\small $2$}\thanks{Corresponding authors} \enspace Hou Pong Chan\raisebox{4pt}{\small $1$} \enspace Long Li\raisebox{4pt} \enspace \; Deng Cai\raisebox{4pt}{\small $1$} \enspace Min Chen\raisebox{4pt}{\small $2$} \enspace Hao Zhang\raisebox{4pt}{\small $3$} \\
\raisebox{4pt}{\small $1$}The Chinese University of Hong Kong \\
\raisebox{4pt}{\small $2$}South China University of Technology \\
\raisebox{4pt}{\small $3$}Nanyang Technological University \\
{\tt weiwen.xuu@gmail.com; Jialiu0330@hust.edu.cn} \\
}

\begin{document}

\maketitle

\begin{abstract}
    Reinforcement learning with verifiable rewards (RLVR) commonly uses entropy for advantage shaping. However, entropy cannot distinguish useful uncertainty from detrimental confusion, limiting its effectiveness as a correctness signal. We propose Contrastive Policy Optimization (\texttt{CPO}), which uses token-level contrastive disagreement between reference-guided and vanilla generation distributions for correctness-aware advantage shaping. Both theoretical and empirical results show that this disagreement reliably indicates token-level correctness. We further show that  On-policy Distillation is a special case of \texttt{CPO}, where the posterior distribution is instantiated by an external teacher model. \texttt{CPO} also resolves the zero-advantage problem. Experiments on in-domain and out-of-domain benchmarks demonstrate that \texttt{CPO} substantially outperforms entropy-based RLVR methods while maintaining strong generalization. Further analysis shows that correct and incorrect responses naturally support exploration and exploitation respectively, and balancing both leads to the best performance.
\end{abstract}

\section{Introduction}
\label{sec:intro}
Reinforcement learning with verifiable rewards (RLVR) has substantially advanced performance in domains like mathematics and programming by providing a simple yet effective recipe for improving reasoning~\cite{lambert2024tulu,yang2025qwen3, guo2025deepseek,team2025kimi,yue2025does}. 
A notable example is Group Relative Policy Optimization (GRPO~\cite{shao2024deepseekmath}), which streamlines the RL formulation by eliminating the need for a separate critic model~\cite{schulman2017proximal}.
GRPO estimates advantages by sampling multiple outputs per prompt and normalizing rewards within each group. 
However, its binary correctness-based reward scheme introduces two fundamental limitations: (1) it neglects the distinct contributions of individual tokens throughout the reasoning trajectory~\cite{cheng2025reasoning}, and (2) it offers no mechanism for differentiating quality among responses in zero-advantage groups, wasting vast training data~\cite{zheng2025act,le2025no}.

% \xww{show the advantage shaping methods}
\begin{figure}[t]
    \centering
    \includegraphics[width=1\linewidth]{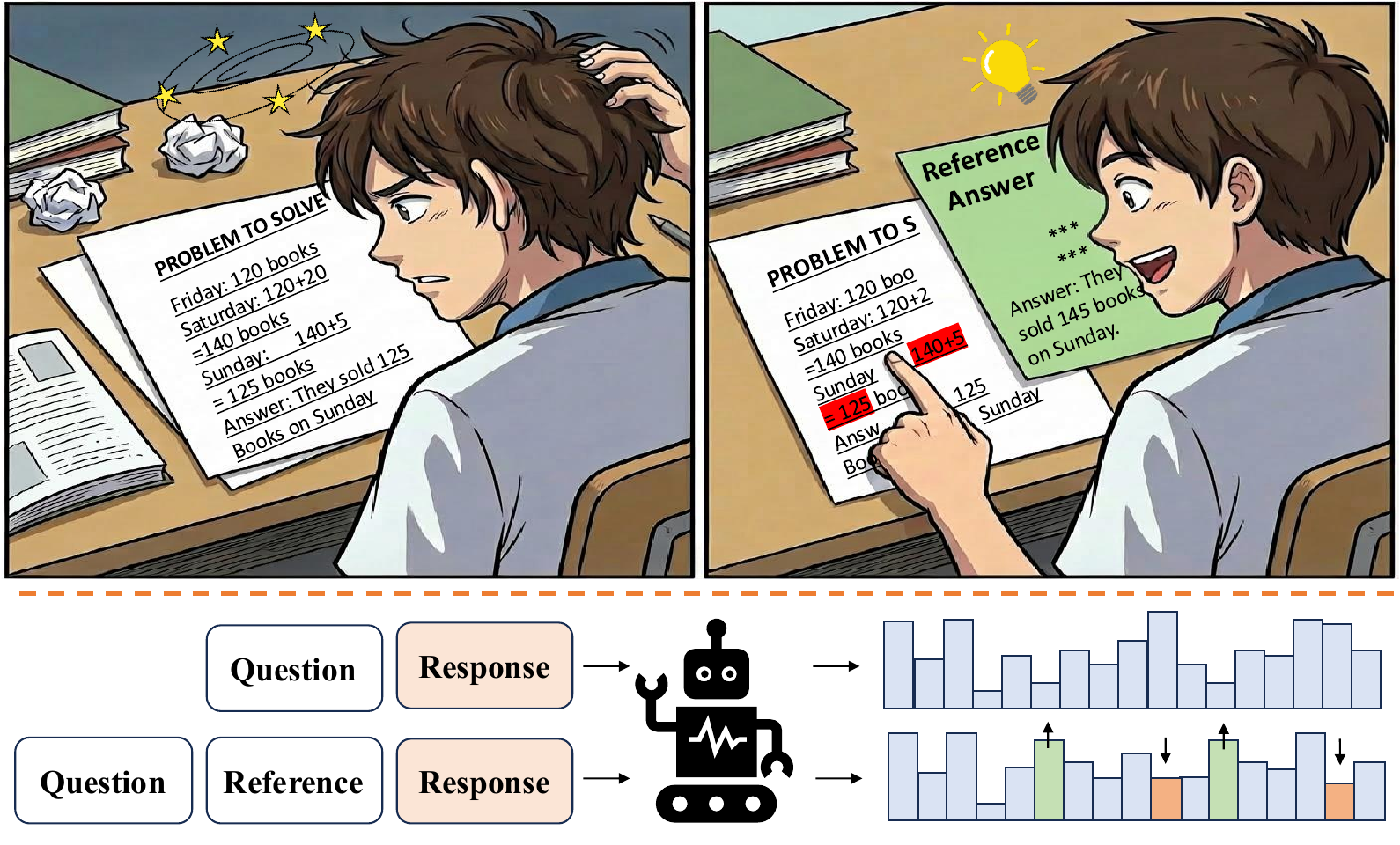}
    \caption{ \textbf{Top}: A student spots his homework mistake by comparing his work to the reference answer. \textbf{Bottom}: An LLM's response generation probabilities shift when guided by a reference answer.}
    \label{fig:intro}
    \vspace{-10pt}
\end{figure}

Recent work has identified \textit{entropy} as a critical signal for advantage shaping in GRPO~\cite{deng2025decomposing}.
Subsequent methods incorporate entropy to refine optimization signals, either by selectively updating high-entropy tokens~\cite{wang2025beyond} or by augmenting advantage estimates with entropy-based terms~\cite{cheng2025reasoning, agarwal2025unreasonable, wang2025harnessing}.
They effectively improve exploration at uncertain decision points and introduce finer-grained advantage differentiation that better distinguishes token importance both within and across trajectories~\cite{le2025no}.
However, entropy quantifies uncertainty while remaining agnostic to the quality of underlying reasoning. A model may exhibit high entropy either \textit{when productively exploring alternative solutions} or \textit{when becoming confused by incorrect paths}~\cite{wang2025beyond}. Thus, entropy and correctness are intrinsically misaligned.
This creates an inherent challenge for optimization: \textit{entropy alone cannot distinguish valuable exploration from detrimental errors}. This competing effects have led to contradictory design choices across existing methods: some reward high-entropy tokens in correct responses~\cite{le2025no}, while others instead favor low-entropy tokens~\cite{wang2025harnessing}.
Fundamentally, what is needed is a \textit{correctness-aware} signal that can reliably identify tokens aligned with correct reasoning.

This intuition can be illustrated through human-problem-solving. As shown in Figure~\ref{fig:intro}, students often struggle to identify mistakes without guidance, but once a reference answer is provided, incorrect steps become immediately apparent upon review.
This contrast highlights a simple principle: 
\textit{correctness emerges through comparison between self-review and reference-guided review}.
A similar pattern appears in LLMs. When evaluating their own generations alone, LLMs lack fine-grained awareness of the erroneous parts. 
In contrast, conditioning on a reference answer induces distinct probability shifts on specific tokens:
increasing token probabilities aligned with the reference and suppressing those that contradict it. Such contrastive probability changes may provide a more faithful, correctness-aligned signal than entropy.

Building on these insights, we propose \textbf{C}ontrastive \textbf{P}olicy \textbf{O}ptimization (\ours), a framework for correctness-aware advantage shaping in RLVR. \ours introduces a \textit{token-level contrastive disagreement} that quantifies the divergence between reference-guided and vanilla generation distributions.
Through theoretical and empirical analysis, we show that this contrastive disagreement reliably indicates token-level correctness.
Since both contrastive disagreement and RLVR rewards are correctness signals operating at different granularities, token-level and 
trajectory-level respectively, their combination is principled rather than heuristic, 
motivating our advantage shaping design. To ensure stability, we constrain the disagreement magnitude to prevent it from overwhelming or reversing the trajectory-level advantage. \ours's token-level advantage shaping directly mitigates the zero-advantage issue prevalent in existing RLVR approaches~\cite{le2025no,yu2025dapo}, enabling more informative gradients and effective exploration across large-scale training corpora.

Interestingly, contrastive disagreement also offers a principled theoretical grounding 
for on-policy distillation (OPD ~\cite{agarwal2024onpolicy,lu2025onpolicydistillation}): the reverse KL used in OPD is a special case of contrastive disagreement.
More broadly, our theoretical framework subsumes diverse OPD variants under a single correctness-driven perspective: whether the teacher is an external model~\cite{lu2025onpolicydistillation}, a self-teacher conditioned on reference answers~\cite{lin2025ravr}, critic feedback~\cite{hubotter2026reinforcement}, or other contextual signals~\cite{wang2026openclaw}, what matters is that the teacher distribution is more correctness-informed than the current policy.

Results show that \ours substantially enhances reasoning capabilities over entropy-based RLVR methods while largely preserving out-of-domain performance. On Qwen2.5-Math-7B and Qwen3-Base-4B, \ours outperforms GRPO by 7.7\% and 8.5\% on average, respectively.
We analyze the distinction between contrastive disagreement and entropy, finding that \ours focuses on discriminative features tied to correctness, whereas entropy-based methods emphasize linguistic variability. We show that correct and incorrect responses in \ours naturally support exploration and exploitation, and that balancing these roles leads to superior performance.
Our contributions are as follows:
(1) We propose contrastive disagreement as a more reliable token-level correctness signal than entropy, supported by both theoretical and empirical evidence.
(2) Our contrastive disagreement lays a theoretical foundation for OPD: diverse OPD 
variants, whether based on external teachers, reference-guided, or critic-guided 
self-teachers, can be understood as different instantiations of a correctness-informed 
distribution, unifying RLVR and OPD under a single correctness-driven objective.
(3) We introduce \ours, a correctness-aware advantage shaping framework that is computationally efficient and effectively addresses the zero-advantage problem in RLVR.
(4) We show \ours's effectiveness on in-domain math reasoning and out-of-domain generalization, along with in-depth analysis of its learning dynamics.

\section{Contrastive Policy Optimization}
\label{sec:method}
We begin by validating the reliability of contrastive disagreement to provide a correctness-aligned signal for token-level advantage shaping, both theoretically and empirically. Building on this foundation, we introduce \ours framework.

\subsection{Contrastive Disagreement} % 推导

\subsubsection{Theoretical Analysis}
% \paragraph{Setup.} 
Given the question $x$,  the sampled trajectory $y=(y_1,\dots,y_{|y|})$\footnote{Without loss of generality, we drop the subscript $i$ in this subsection. Here $y_t \equiv y_{i,t}$ denotes the token at position $t$ of $y_i$.} from the policy $\pi$ and the environment reward $R(x,y)\in\{-1,1\}$, 
we define the \textit{position-wise probability of the oracle correctness event} $C := \{R(x,Y)=1\}$:
\begin{equation}
\small
\begin{split}
g(x,y_{<t},y_t) := \mathbb{P}\big(C | X=x, Y_{<t}=y_{<t}, Y_t=y_t \big),
\end{split}
\end{equation}
which measures the probability that choosing token $y_t$ at position $t$ will lead to a correct completion, averaging over all possible future continuations.

% \paragraph{Correctness-conditioned token distribution.}
We define the ideal next-token distribution conditioned on the oracle correctness as:
\begin{equation}
\small
\tilde{\pi}_{\mathrm{post}}(y_t\mid x,y_{<t}) := \mathbb{P}\!\left( Y_t=y_t \,\middle|\, X=x,\; Y_{<t}=y_{<t},\; C \right).
\end{equation}
By Bayes' rule, this distribution is regarded as a reweighting of the prior by correctness probability:
\begin{equation}
\small
\tilde{\pi}_{\mathrm{post}}(y_t\mid x,y_{<t}) =  \frac{\mathbb{P}\big(Y_t=y_t  | X=x, Y_{<t}=y_{<t}\big)\, g(x,y_{<t},y_t)}{Z(x,y_{<t})}  = \frac{\pi(y_t\mid x,y_{<t})\, g(x,y_{<t},y_t)}{Z(x,y_{<t})},
\label{eq:bayes_reweight}
\end{equation}
where $Z(x,y_{<t}) = \sum_b \pi(b\mid x,y_{<t})\, g(x,y_{<t},b)$ is the prior-average correctness probability at prefix $y_{<t}$.
Taking logarithms yields:
\begin{equation}
\small
\begin{split}
\log \frac{\tilde{\pi}_{\mathrm{post}}(y_t\mid x,y_{<t})}{\pi (y_t\mid x,y_{<t})}= \log g(x,y_{<t},y_t) - \log Z(x,y_{<t}).
\end{split}
\label{eq:ideal_disagreement}
\end{equation}
Crucially, for fixed $(x,y_{<t})$, this log-ratio is monotonically increasing in $g(x,y_{<t},y_t)$. Therefore:
\begin{itemize}[leftmargin=*,topsep=1pt]
\setlength{\itemsep}{0pt}
\setlength{\parskip}{0pt}
\setlength{\parsep}{0pt}
\item Tokens with \textbf{below-average} correctness  ($g(x,y_{<t},y_t) < Z(x,y_{<t})$) have \textbf{negative} log-ratio.
\item Tokens with \textbf{above-average} correctness  ($g(x,y_{<t},y_t) > Z(x,y_{<t})$) have \textbf{positive} log-ratio.
\end{itemize}

% \paragraph{Connecting reference-guided scoring to correctness.}
In practice, we do not have direct access to the ideal distribution $\tilde{\pi}_{\mathrm{post}}$. 
However, we can condition the model on the oracle reference $y^\star$ (\textit{e.g.},  prompting the model to refine based on $y^\star$) to obtain $\pi_{\mathrm{post}}(y_t\mid x,y^\star,y_{<t})$.
Since $y^\star$ encodes the ground-truth that determines correctness and the prompt also asks the model to refine, conditioning on such prompt serves as a practical proxy for conditioning on the correctness event $C$, \textit{i.e.}, $\pi_{\mathrm{post}}(\cdot\mid x,y^\star,y_{<t}) \approx \tilde{\pi}_{\mathrm{post}}(\cdot\mid x,y_{<t})$.
Consequently, the \textit{contrastive disagreement}
\begin{equation}
\small
\begin{aligned}
\delta_t(x,y) &:= \log \frac{\pi_{\mathrm{post}}(y_t\mid x,y^\star,y_{<t})}{\pi (y_t\mid x,y_{<t})} \propto \log g(x,y_{<t},y_t)
\end{aligned}
\label{eq:delta}
\end{equation}
provides a good estimate of the ideal log-ratio in Eq.~\eqref{eq:ideal_disagreement}.
Importantly, this holds \textit{regardless of whether the rollout $y$ is correct or not}---the correctness probability $g(x,y_{<t},y_t)$ is defined over future continuations, not individual trajectories. Full derivations are provided in Appendix~\ref{apdx:sec:theory}.

% Large negative disagreement $\delta_t(x,y) \ll 0$ identifies tokens that are suppressed under oracle-informed guidance, indicating they are likely to be \textit{incorrect-leaning} at that prefix; conversely, large positive disagreement $\delta_t(x,y) \gg 0$ identifies tokens that are promoted, indicating they are likely to be \textit{correct-leaning}.
% This provides a principled token-level signal for advantage shaping: \textit{we penalize tokens with negative disagreement and reward tokens with positive disagreement, thereby encouraging the model to favor tokens that increase correctness probability}. 

\subsubsection{Empirical Studies}
\label{sec:emp_study}
\begin{wrapfigure}{r}{0.5\linewidth}
\vspace{-5pt}
    \centering
    \includegraphics[width=\linewidth]{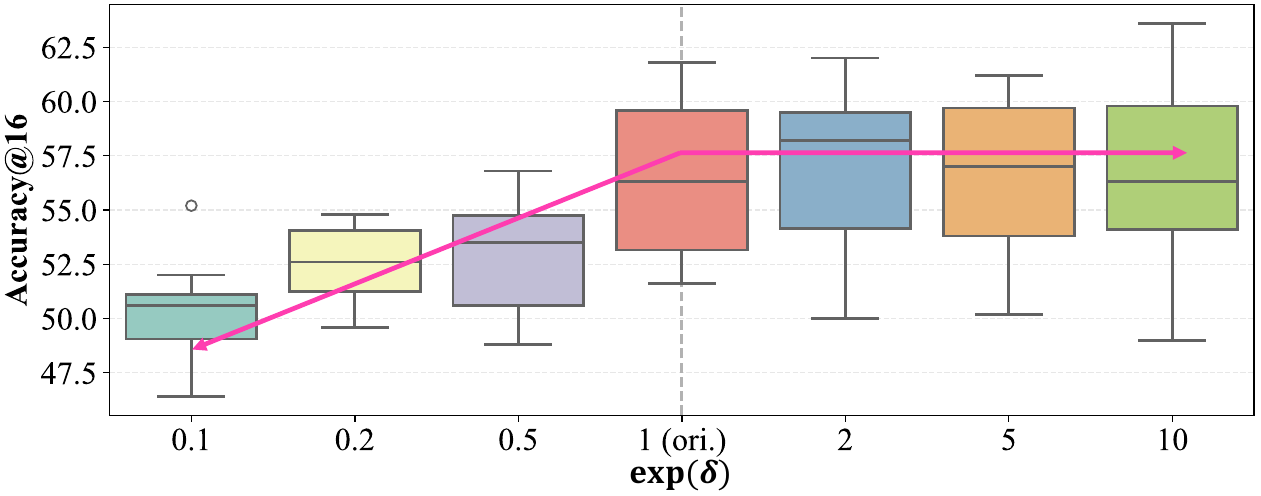}
    \vspace{-18pt}
    \caption{Impact of contrastive disagreement on token correctness (Accuracy@16). }
    \label{fig:pre}
    \vspace{-16pt}
\end{wrapfigure}

% Contrastive Generation Disagreement indicates Incorrectness

We empirically validate Eq.~\eqref{eq:delta}, showing that contrastive disagreement $\delta_t(x,y)$ reliably indicates token-level correctness $g(x,y_{<t},y_t)$. Since $g(x,y_{<t},y_t)$ represents the expected correctness averaging over all possible future continuations from token $y_t$, it is intractable to compute exactly. We approximate it with Accuracy@16, \textit{i.e.}, sampling 16 continuations from $y_t$ and measuring their average score.

\begin{figure*}[t]
    \centering
    \includegraphics[width=\linewidth]{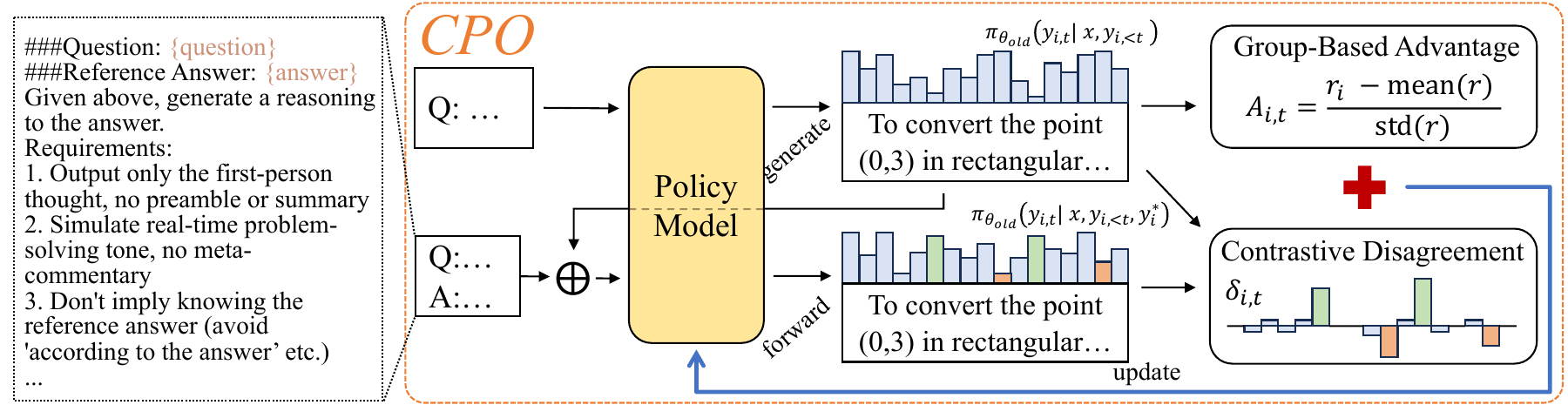}
    \caption{The illustration of \ours. Prior and reference-guided posterior likelihoods are contrasted to shape token-level advantages.  }
    \label{fig:pipeline}
    \vspace{-3pt}
\end{figure*}

\textbf{Experimental setup.}
We adopt Qwen2.5-7B-Math~\cite{yang2024qwen2} as our prior policy $\pi$ and evaluate on queries from Math-500~\cite{hendrycks2021measuring}. 
For each query $x$, we first sample a vanilla generation $y \sim \pi (\cdot \mid x)$ and record the token-level generation probabilities $\pi (y_{t} \mid x, y_{<t})$ for each token $y_{t}$. 
Next, we compute reference-guided generation probabilities $\pi_{\mathrm{post}}(y_{t} \mid x, y^\star, y_{<t})$ by forward-passing the sampled sequence $y$ through the model conditioned on the reference $y^\star$.
Both vanilla and reference-guided prompts are in Appendix~\ref{apdx:sec:prompt}.
We then compute the token-level contrastive disagreement $\delta_t(x,y)$.

% \paragraph{Identifying fork tokens.}
We identify the first token position $t^\star$ where the disagreement deviates from a predefined threshold $\delta$:
\begin{equation}
\small
t^\star = \begin{cases}
\min \left\{ t \in \{1, 2, \ldots, |y|\} \mid \delta_t <  \delta \right\}, & \delta < 0, \\
\min \left\{ t \in \{1, 2, \ldots, |y|\} \mid \delta_t >  \delta \right\}, & \delta > 0.
\end{cases}
\end{equation}
Note that  $\delta_t = 0$ denotes the boundary where no disagreement exists. For $\delta <0$, more negative values indicate greater incorrectness, while for $\delta >0$, more positive values indicate greater correctness.

% \paragraph{Resampling continuations.}
We then resample continuations starting from these high-disagreement positions: $\hat{y} \sim \pi(\cdot \mid x, y_{\le t^\star}) $. 
By conditioning on the prefix up to $t^\star$ (which includes the high-disagreement token $y_{t^\star}$), each resampled completion $\hat{y}$ is directly influenced by the token choice at this critical fork point. We repeat this resampling process 16 times for each query and compute the Accuracy@16 as the empirical approximation of the correctness probability $g(x,y_{<t^\star},y_{t^\star})$.

 % If our theoretical analysis holds—that $\delta_t(x,y)$ correlates with $g(x,y_{<t},y_t)$—then tokens with larger negative disagreement should yield lower accuracy@16, while tokens with larger positive disagreement should yield higher accuracy@16.

\textbf{Results.}
Figure~\ref{fig:pre} shows the Accuracy@16 of $\hat{y}$ as a function of the threshold $\exp(\delta)$. 
The results show clear asymmetry. As we select tokens with increasingly negative disagreement ($\exp(\delta)<1$), accuracy drops substantially from 57\% at the boundary ($\exp(\delta) = 1$) to approximately 48\% at $\exp(\delta) = 0.1$. This strongly confirms that negative contrastive disagreement reliably identifies incorrect-leaning tokens.
In contrast, the positive disagreement ($\exp(\delta) > 1$) maintains stable accuracy near baseline level.
This asymmetry is expected: selecting correct-leaning tokens merely maintains the baseline performance, which is fundamentally bounded by the backbone model's capabilities.
Nevertheless, the results validate our theoretical analysis, showing that  $\delta_t(x,y)$ correlates with $g(x,y_{<t},y_t)$. Notably, this correlation is particularly strong for incorrect tokens, confirming the theoretical findings in a more realistic scenario.
% This observation provides valuable guidance for advantage shaping: penalizing tokens with negative disagreement offers a robust and actionable correctness-aligned signal for policy optimization.
% This asymmetry is expected: selecting correct tokens merely preserves baseline performance, bounded by the backbone model's capabilities. This guides our advantage shaping: penalizing tokens with negative disagreement provides a robust correctness signal for policy optimization.

% \textbf{Implication.}
% Recall the contrastive disagreement in Eq.~\eqref{eq:delta}: it is similar to the reverse KL used in OPD, where OPD can be recovered by simply instantiating $\pi_{\mathrm{post}}$ as an external teacher model. However, \ours proposes a more general theoretical perspective: \textit{any distribution that is more correctness-informed than the current policy serves as a valid instantiation of} $\pi_{\mathrm{post}}$, including an external teacher, the policy conditioned on reference answers~\cite{lin2025ravr}, or critic feedback~\cite{xu-etal-2024-reasons}, or any other distribution that is more likely to produce correct outputs than the current policy. Moreover, under our framework, contrastive disagreement and RLVR rewards are both correctness signals operating at different granularities, token-level and trajectory-level, making their combination in \ours principled rather than heuristic. This motivates our following advantage-shaping design.

\textbf{Implication.}
The contrastive disagreement in Eq.~\eqref{eq:delta} resembles the reverse KL used in OPD, which can be viewed as a special case by with an external teacher as $\pi_{\mathrm{post}}$. More generally, \ours admits any distribution more correctness-informed than the current policy as $\pi_{\mathrm{post}}$, including external teachers, critic-guided policies~\cite{xu-etal-2024-reasons}, or reference-conditioned policies~\cite{lin2025ravr}. In this work, we use reference answers because they provide direct correctness signals. We also show that reference-conditioned policies can even outperform external teachers, while avoiding reliance on stronger models and enabling self-improvement within the same policy in  Sec~\ref{sec:opd_to_cpo}. 

\subsection{Advantage Shaping with Contrastive Disagreement}

Under the above view, contrastive disagreement and RLVR rewards serve as complementary correctness signals at token and trajectory levels, motivating our advantage-shaping design.
Since trajectory-level advantage $A_i$ cannot distinguish correct from erroneous tokens, we use contrastive disagreement  $\delta_t(x,y)$ to shape advantages for finer-grained credit assignment.

The overall framework is depicted in Figure~\ref{fig:pipeline}. The policy model generates the response with prior likelihood $\pi_{\theta_{\text{old}}}(y_t\mid x,y_{<t})$. Then, we concatenate it with a reference-guided prompt, detailed in Appendix~\ref{apdx:sec:prompt}, which asks the model to regenerate a first-person response based on the reference answer. This prompt simulates real-time problem-solving tone, making the sampled response an appropriate continuation. 
This combined sequence is resent to the policy model to obtain the posterior likelihood $\pi_{\theta_{\text{old}}}(y_t\mid x, y^\star,y_{<t})$. Following standard RLVR methods~\cite{shao2024deepseekmath}, we calculate the trajectory-level advantage $A_{i,t} = A_i$, and additionally compute the contrastive disagreement as posterior-prior KL:
\begin{equation}
\small
\label{eq:cpo_delta}
\delta_{i,t} = \log \frac{\pi_{\theta_{\text{old}}}(y_{i,t}\mid x, y_i^\star,y_{i,<t}) }{\pi_{\theta_{\text{old}}}(y_{i,t}\mid x,y_{i,<t})} \text{.}
\end{equation}
\textbf{Advantage Magnitude and Direction.}
The trajectory-level advantage ${A}_{i,t}$ should be the primary determinant, as the goal is to reward the generation of correct solutions. Meanwhile, the disagreement magnitude $|\delta_{i,t}|$ indicates how decisively a token leans toward correctness or incorrectness. We therefore propose the base form of advantage shaping as ${A}_{i,t} + \alpha \cdot \delta_{i,t}$, where $\alpha > 0$ controls the strength of the token-level signal. 
The primacy of trajectory-level advantage ${A}_{i,t}$ extends beyond magnitude to direction as well.
However, in practice, conflicts can arise between trajectory-level and token-level signals. For instance, a good response (${A}_{i,t} > 0$) may contain tokens with strongly negative disagreement ($\delta_{i,t} \ll 0$), potentially reversing the advantage sign. 
To ensure the shaped advantage direction remains consistent with ${A}_{i,t}$, we introduce a clipping operation:
\begin{equation}
\small
\tilde{A}_{i,t} = \begin{cases}
{A}_{i,t} + \max(\alpha^{+} \cdot \delta_{i,t}, -\frac{1}{2}{A}_{i,t} ), & {A}_{i,t} > 0, R_i=1, \\
{A}_{i,t} + \min(\alpha^- \cdot \delta_{i,t}, -\frac{1}{2}{A}_{i,t} ), & {A}_{i,t} < 0, R_i=-1, \\
\alpha^+ \cdot \delta_{i,t}, & {A}_{i,t} = 0, R_i=1, \\
\alpha^- \cdot \delta_{i,t}, & {A}_{i,t} = 0, R_i=-1, \\
\end{cases}
\label{eq:shaped_advantage}
\end{equation}
where $\alpha^{+}$, $\alpha^{-}$ denote the combination strength for correct and incorrect responses. 
We do not clip the original zero-advantage responses, because the original advantage offers no value for policy optimization. However, our token-level disagreements still offer meaningful learning signals, enabling more informative gradients and more effective exploration across large-scale training corpora.
% Zero-advantage responses are typically clipped, since the original advantage carries no informative signal for policy optimization. In contrast, we retain zero-advantage responses since our token-level contrastive disagreement still offers meaningful supervision, leading to more informative gradients and more effective exploration at scale.
The final CPO objective follows the standard GRPO, with $\tilde{A}_{i,t}$ replacing the original advantage:
\begin{equation}
\small
\begin{split}
\label{eq:cpo}
\mathcal{J}_{\text{CPO}}(\theta) = \mathbb{E}_{x \sim \mathcal{D}, y_i \sim \pi_{\theta_{\text{old}}}(\cdot|x)} \Bigg[ \frac{1}{G} \sum_{i=1}^{G} \frac{1}{|y_i|} \sum_{t=1}^{|y_i|} \min \left( r_{i,t}(\theta) \textcolor{red}{\tilde{A}_{i,t}}, \text{clip}(r_{i,t}(\theta), 1-\varepsilon, 1+\varepsilon) \textcolor{red}{\tilde{A}_{i,t}} \right) \Bigg] \text{.}
\end{split}
\end{equation}

% \subsection{CPO: Connecting RLVR with On-Policy Distillation}
% %\subsection{CPO and Its Connection to On-Policy Distillation}
% % \subsection{CPO as On-Policy Distillation}
% Let's revisit the base form of \ours (without clipping):
% \begin{equation}
%    \tilde{A}_{i,t} = {A}_{i,t} + \alpha \cdot \log \frac{\pi_{\theta_{\text{old}}}(y_{i,t}\mid x, y_i^\star,y_{i,<t}) }{\pi_{\theta_{\text{old}}}(y_{i,t}\mid x,y_{i,<t})},
% \end{equation}
% where the first term represents standard RLVR, determining whether to reinforce or suppress the response. The second term is precisely the reverse KL used in on-policy distillation (OPD;~\citebare{lu2025onpolicydistillation}), offering token-level teacher guidance.
% The key distinction between our second term and OPD is the teacher construction. While OPD relies on an external strong model, \ours adopts a self-distillation framework where the reference-conditioned policy $\pi_{\theta_{\text{old}}}(\cdot\mid x, y^\star,y_{<t})$ serves as the teacher.
% This works because conditioning on the reference answer $y^\star$ makes the posterior policy knowledgeable enough to solve this problem. Moreover, self-distillation design eliminates external teacher models, enhancing the generality and accessibility of \ours.
% Nevertheless, employing a stronger external teacher may improve token-level guidance, which we leave for future work.
% As a result, \ours offer a natural combination of RLVR and OPD.

\section{Experiment}
\subsection{Experimental Setup}
\label{sec:exp_setup}
\paragraph{Training and evaluation.}
During training, we conduct experiments on two base models: the general-purpose Qwen3-Base-4B~\cite{yang2025qwen3} and the domain-specific Qwen2.5-Math-7B~\cite{yang2024qwen2}. We train models on the MATH dataset~\cite{hendrycks2021measuring} with 7.5k problems. We implement \ours based on GRPO~\cite{shao2024deepseekmath}. Full hypeparameters are in Appendix~\ref{apdx:sec:impl}.
% During training, we use a prompt batch size of 1024 with 8 rollouts per prompt, and set the sampling temperature to 1.0. The maximum context length is set to 4096 for Qwen2.5-Math-7B and 8192 for Qwen3-Base-4B. The combination strength $\alpha^{+}$ and $\alpha^{-}$ are both set to 0.025. We perform policy updates with a mini-batch size of 256 and a learning rate of 1e-6. The KL penalty coefficient between the policy and reference models is set to 1e-3, and the PPO clip ratio is set to 0.2. We conduct all experiments on 8 A800 GPUs. The additional forward pass required to calculate reference-guided generation probabilities increases training time by approximately 20\%.
We evaluate on both in-domain and out-of-domain benchmarks.
In-domain math reasoning tasks include MATH500 \cite{hendrycks2021measuring}, AIME2024/2025~\cite{maa2025aime}, and AMC23~\cite{maa2023amc}.
For out-of-domain evaluation, we use GPQA \cite{rein2024gpqa}, MMLU-Pro \cite{wang2024mmlupro}, and Knowlogic \cite{zhan2025knowlogic}, covering knowledge-driven and logical reasoning tasks. 
% We use Pass@16 as our evaluation metrics.
In inference, we set temperature to 0.6 and top-p to 0.95 
for sampling, and report Pass@16 as our primary evaluation metric.

\textbf{Baselines.}
We compare \ours with two baseline categories~\footnote{We also compare \ours with SFT and distillation-based methods in 
Appendix~\ref{sec:apdx:sft}, where all methods utilize the same reference answer during training for fair comparison.}:
 \textbf{Standard RLVR methods} including \textit{GRPO}~\cite{shao2024deepseekmath} and \textit{DAPO}~\cite{yu2025dapo}.
\textbf{Entropy-intervened methods} that incorporates entropy into GRPO: (1) \textit{Entropy-Tokens} \cite{wang2025beyond}, which optimizes tokens only in the top 20\% entropy range; (2) \textit{Entropy-Adv} \cite{cheng2025reasoning}, which directly augments the advantage with an entropy term; (3) \textit{EM-RL-token} \cite{agarwal2025unreasonable}, directly uses negative entropy as rewards in RL; (4) \textit{RL-ZVP} \cite{le2025no}, which applies entropy regularization exclusively to zero-advantage prompts for better data utilization; and (5) \textit{IB-reg} \cite{lei2025revisiting}, which modulates entropy regularization based on advantage values.

\subsection{Main Results}
\label{sec:main_results}

\begin{table*}[t]
\centering
\caption{ Qwen2.5-Math-7B experiments with PASS@16 on math and out-domain tasks. Bold numbers denote the best results per task.}
\label{tab:qwen2.5}
\adjustbox{max width=\textwidth}{
\setlength{\tabcolsep}{1mm}{
\begin{tabular}{l||cccc|ccc|c}
\toprule
\textbf{Methods} & \textbf{MATH-500} & \textbf{AIME2024} & \textbf{AIME2025} & \textbf{AMC2023} & \textbf{GPQA}  & \textbf{MMLU-PRO} & \textbf{KnowLogic} & \textbf{Avg.}\\
\midrule
Original & 87.8 & 43.3 & 30.0 & 85.0 & \textbf{70.5} & 63.4 & 55.8 & 62.3 \\
GRPO & 88.2 & 43.3 & 30.0 & 82.5 & 64.5 & 60.7 & 68.2 & 62.5 \\
DAPO & 88.2 &43.3 &33.3 & 90.0 & 66.3 & 64.6& 61.1 & 63.8\\
% GSPO \\
Entropy-Tokens & 85.0 & 53.3 & 26.7 & 77.5 & 58.5 & 54.8 & 66.2 & 60.3\\
% grpo +first20 & 88.00 & 53.33 & 33.33 & 85.00 & 57.14 \\
Entropy-Adv & 84.8 & 43.3 & 33.3 & 77.5 & 55.4 & 53.9 & 68.2 & 59.5\\
EM-RL-token & 88.8 & 36.7 & 30.0 & 90.0 & 60.8  & 64.0 & 66.3 & 62.4\\
RL-ZVP & 86.6 & 50.0 & 30.0 & 77.5 & 56.7 & 56.0 & 68.2 & 60.7\\
IB-reg & 88.0 & 60.0 & 36.7 & 80.0 & 58.7 & 55.3 & 68.5 & 63.9\\

% GCPO & 88.4 & 53.3 & 36.7 & 85.0 & 61.4 & 59.2 & 64.3 & 64.0\\
% PRIME & 87.8  &60.0 & 33.3 & 87.5 & 52.4 & 53.1 & 65.7 & 62.8\\
% VL-Cogito & 88.8 & 53.3 & 36.7 & 90.0 & 62.6 & 62.1 & 66.7 & 65.7 \\ 
% W-REINFORCE & 88.8 & 53.3 & 30.0 & 85.0 & 68.3 & 50.7 &54.6 & 61.5 \\
\midrule
\ours & \textbf{89.0} & \textbf{60.0} & \textbf{43.3} & \textbf{92.5} & 68.8 & \textbf{69.0} & \textbf{68.9} & \textbf{70.2}\\
\bottomrule
\end{tabular}}}
\vspace{-5pt}
\end{table*}

We present comparative results in Table~\ref{tab:qwen2.5} and Table~\ref{tab:qwen3} for Qwen2.5-math-7B and Qwen3-Base-4B.
\ours achieves superior performance across both backbones, outperforming GRPO by 7.7\% and 8.5\%, respectively.
For in-domain math reasoning, \ours yields particularly notable gains on AIME2025, the most challenging benchmark requiring high per-step accuracy over long reasoning chains. This validates that inducing fine-grained supervision signals is particularly well-suited for complex reasoning-intensive tasks.
More strikingly, while baselines show substantial out-of-domain degradation (except KnowLogic on Qwen2.5-Math-7B), \ours preserves and even improves generalization, exceeding Qwen3-Base-4B, on GPQA and MMLU-PRO by 1.5\% and 2.8\%, respectively.
This likely stems from its self-contrastive correctness signals that avoid external preference biases, while token-level advantage shaping reduces overfitting by targeting critical tokens rather than penalizing all tokens uniformly. Due to computational constraints, we report mean $\pm$ std over 3 independent runs for CPO, GRPO and  the strongest baseline DAPO in Appendix~\ref{sec:apdx:statis},  showing that CPO consistently outperforms both baselines with improvements well beyond 
one standard deviation.

\subsection{Unifying OPD into CPO via Posterior Design}
\label{sec:opd_to_cpo}
\begin{wrapfigure}{r}{0.5\linewidth}
    \centering
    \vspace{-14pt}
    \includegraphics[width=1\linewidth]{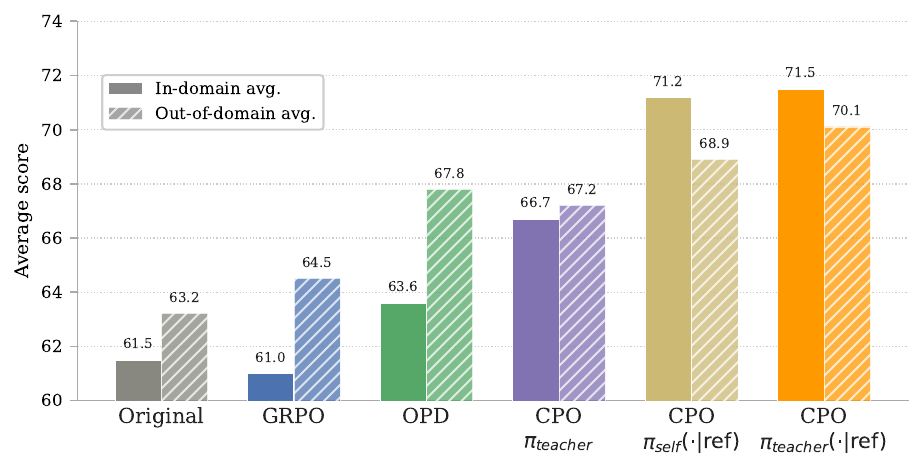}
    \vspace{-20pt}
    \caption{Comparison of different posterior designs ($\pi_{\text{post}}$) under a unified CPO framework.}
    \label{fig:opd}
    \vspace{-7pt}
\end{wrapfigure}

We study different instantiations of $\pi_{\text{post}}$ under the CPO framework in Figure~\ref{fig:opd}. OPD employs a stronger teacher, \textit{Qwen2.5-Math-72B}, to guide the student via a reverse KL objective. OPD clear out-of-domain gains but limited in-domain improvements.
We then consider three variants: (1) \textbf{CPO}-${\pi_{\text{teacher}}}$, using the teacher directly as $\pi_{\text{post}}$; (2) \textbf{CPO}-$\pi_{\text{self}}(\cdot|ref)$, our method that conditions the current policy on reference answers; and (3) \textbf{CPO}-$\pi_{\text{teacher}}(\cdot|ref)$, conditioning the teacher on reference answers. \textbf{CPO}-$\pi_{\text{teacher}}$ underperforms the reference-based variant, suggesting that teacher guidance alone show weaker instance-level correctness signals. In contrast, reference conditioning provides more direct correctness to serve as $\pi_{\text{post}}$. Combining both yields the best performance, where \textbf{CPO}-$\pi_{\text{teacher}}(\cdot|ref)$) consistently achieves the strongest results, indicating that teacher generalization and reference-based correctness signals are complementary.

\begin{table*}[t]
\centering
\caption{Qwen3-Base-4B experiments with PASS@16 on math and out-domain tasks.  }
\label{tab:qwen3}
\adjustbox{max width=\textwidth}{
\setlength{\tabcolsep}{1mm}{
\begin{tabular}{l||cccc|ccc|c}
\toprule
\textbf{Model} & \textbf{MATH-500} & \textbf{AIME2024} & \textbf{AIME2025} & \textbf{AMC2023} & \textbf{GPQA}  & \textbf{MMLU-PRO} & \textbf{KnowLogic} & \textbf{Avg.}\\
\midrule
Original & 87.2 & 23.3 & 30.0 & 90.0 & 76.8 &79.8 & \textbf{94.1} & 68.7\\
GRPO & 91.0 & 53.3 & 40.0 & 92.5 & 56.7 &60.7 & 85.1 & 68.5\\
DAPO & 87.2 & 43.3 &36.6 &92.5 & 75.6 &79.0& 86.7  & 71.6\\
% GSPO \\
Entropy-Tokens & 90.1 & 53.3 & 46.7 & 90.0 & 61.8 &62.5 & 85.4 & 70.0 \\
Entropy-Adv & 88.8 & 50.0 & 40.0 & 85.0 & 57.6 &59.5& 83.6 & 66.4\\
EM-RL-token & 88.6 & 53.3 & 46.7 & 85.0 & 58.9 &62.6& 82.6 & 68.2\\
RL-ZVP & 91.0 & \textbf{56.7} & 40.0 & 92.5 & 64.5 & 62.6 & 85.2 & 70.4\\
IB-reg & 91.4 & 53.3 & 46.7 & 90.0 & 62.1 & 60.9 & 84.0 & 69.8\\

% GCPO & 91.6 & 46.7 & 50.0 & \textbf{95.0} & 70.1 & 70.2 & 85.6 & 72.7 \\
% PRIME& 90.0 & 43.3 & 33.3 & 90.0 & 76.6 & 79.3 \\
% VL-Cogito & 91.2 & 46.7 & 46.7 & 90.0 & 68.8 & 77.8 & 84.7 & 72.3\\ 
\midrule
\ours & \textbf{91.6} & 53.3 & \textbf{53.3} & 92.5 & \textbf{78.3} & \textbf{82.6} & 87.6  & \textbf{77.0}\\
\bottomrule
\end{tabular}}}
\vspace{-9pt}
\end{table*}

\subsection{Ablation Study}
\label{sec:ablation}

\begin{table*}[t]
\centering
\caption{Ablation study of \ours. Full \ours achieves the best in-domain and out-domain performance.}
\label{tab:ablation}
\adjustbox{max width=\textwidth}{
\setlength{\tabcolsep}{1mm}{
\begin{tabular}{l||cccc|ccc}
\toprule
\textbf{Methods} & \textbf{MATH-500} & \textbf{AIME2024} & \textbf{AIME2025} & \textbf{AMC2023} & \textbf{GPQA}  & \textbf{MMLU-PRO} & \textbf{KnowLogic} \\
\midrule
% Original & 87.8 & 43.3 & 30.0 & 85.0 & 70.5 & 63.4 & 55.8\\
\ours & 89.0 & 60.0 & 43.3 & 92.5 & 68.8 & 69.0 & 68.9\\
\midrule
\multicolumn{8}{c}{\small \texttt{Data Effectiveness}}\\
\midrule
$\texttt{CPO}_{neg}$: Shaping incorrect responses  &  89.6 & 60.0 & 40.0 & 92.5 & 65.4 & 64.8 & 68.5\\
$\texttt{CPO}_{pos}$: Shaping correct responses & 88.6 & 56.6 & 33.3 & 92.5 & 66.1 & 66.3 & 68.7\\
\midrule
\multicolumn{8}{c}{\texttt{Design of Reference-guided Prompt}}\\
\midrule
2-turn gold-shot & 89.4 & 50.0 & 40.0 & 90.0 & 65.6 & 62.9 & 66.8\\
2-turn 1-shot & 88.6 & 50.0 & 36.6 & 90.0 &60.0 & 59.3 & 66.1\\
\midrule
\multicolumn{8}{c}{\texttt{Design of advantage}}\\
\midrule
Weighted disagreement - $\pi_{\text{post}}\log \frac{\pi_{\text{post}}}{\pi}$ & 89.2 & 50.0 & 33.3 & 90.0 & 67.6 & 68.1 &69.3\\
Traj-level advantage - $A_i + \frac{1}{|y|}\sum_{t=1}^{|y|}\delta_{i,t}$& 89.6 & 53.3 & 43.3 &  95 & 56.5 & 61.8 & 66.2 \\
w/o advantage clip & 90.2 & 50.0 & 26.6 &92.5 & 67.6 & 67.8 & 65.5\\
Disagreement regularization~\cite{lin2025ravr} &90.2 & 50.0 & 26.6 &85.0 &60.0 & 63.1 & 62.5 \\
\midrule
\bottomrule
\end{tabular}}}
\vspace{-8pt}
\end{table*}

We conduct ablation studies to investigate the components design in \ours in Table~\ref{tab:ablation}. We focus on three key aspects: \textit{data effectiveness}, \textit{reference-guided prompt design}, and \textit{advantage design}.

\textbf{Data Effectiveness.} We examine how response correctness affects \ours. Applying \ours  only on incorrect responses maintains most in-domain math reasoning but reduces out-of-domain generalization (e.g., $-$4.2\% on MMLU-PRO). In contrast to the incorrect-only setting,  shaping only correct responses shows the opposite trend.
This suggests that negative examples primarily drive in-domain improvement by targeting the model's weaknesses. Meanwhile, positive examples play a crucial role in maintaining out-of-domain capabilities. A plausible explanation is that shaping positive examples encourages the model to explore diverse solution paths, thereby preventing overfitting to narrow reasoning patterns and preserving broader generalization.

\textbf{Reference-guided Prompt Design.}
Besides the default single-turn first-person prompt, we test two alternatives: (1) \textit{2-turn gold-shot}, a two-turn structure where the question-reference pair forms the first turn and the model's response the second; (2) Motivated by the few-shot prompting~\cite{brown2020language}, \textit{2-turn 1-shot}, which replaces the first turn with a fixed training question-reference pair. 
% Specifically, the python format is: \texttt{[\{"role": "user", "content": question\}, \{"role": "assistant", "content": reference\}, \{"role": "user", "content": question\}, \{"role": "assistant", "content": model\_response\}]}.
Full prompt formats are in Appendix~\ref{apdx:sec:prompt}.
Results show that our default prompt works best. The improvement of \ours over \textit{2-turn gold-shot} may stem from stronger single-turn problem-solving capabilities in current LLMs.
Meanwhile, \textit{2-turn gold-shot} outperforms \textit{2-turn 1-shot}, suggesting that using the reference to the current question induces more reliable contrastive signals than using an unrelated example.

\begin{figure*}[t]
    \centering
    \includegraphics[width=0.95\linewidth]{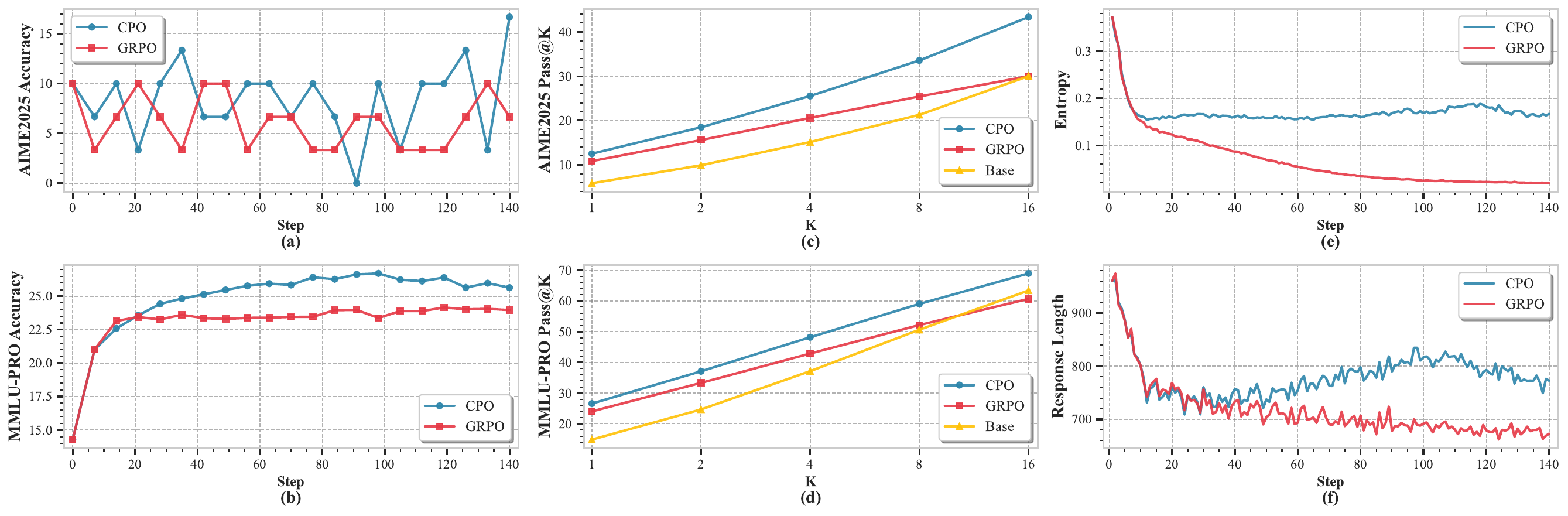}
    \vspace{-15pt}
    \caption{Training dynamics of Qwen2.5-Math-7B under CPO and GRPO. (a-b) Greedy decoding accuracy on in-domain AIME2025 and out-domain MMLU-PRO across training steps; (c-d) Pass@K performance (K=1 to 16) on both test sets, with Qwen2.5-Math-7B results (Base) included for comparison; (e) Training entropy and (f) average response length over the course of training.}
    \label{fig:dynamics}
    % \vspace{-1pt}
\end{figure*}

\begin{figure*}[t]
    \centering
    \includegraphics[width=0.95\linewidth]{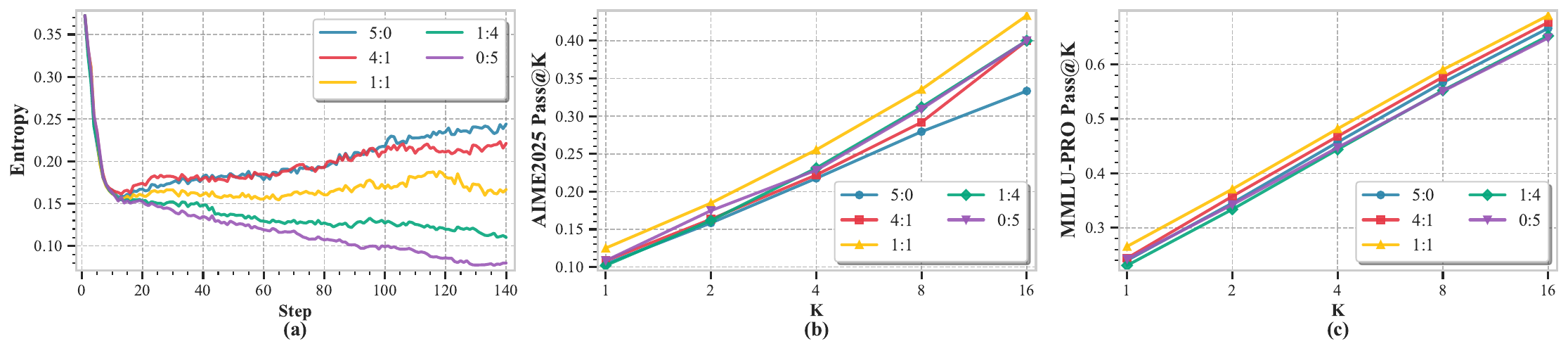}
    \vspace{-15pt}
    \caption{Training dynamics of CPO with different $\alpha^{+}:\alpha^{-}$ ratios. (a) training entropy; (b)/(c)  Pass@$K$ on AIME2025 and MMLU-PRO. }
    \label{fig:cpo_ratio}
    \vspace{-8pt}
\end{figure*}

\textbf{Advantage Design.} We investigate advantage shaping methods to identify the most effective one. 
\begin{itemize}[leftmargin=*,topsep=0pt]
\setlength{\itemsep}{0pt}
\setlength{\parskip}{0pt}
\setlength{\parsep}{0pt}
    \item  \textit{Weighted disagreement} $\pi_{\text{post}}\log \frac{\pi_{\text{post}}}{\pi}$. $\pi_{\text{post}}$ is designed to prevent the Eq.~\ref{eq:delta} from exploding when both $\pi$ and $\pi_{\text{post}}$ are very small ( $\pi \ll \pi_{\text{post}} \ll 0.1$). 
    Compared to the defaulted \ours, the weighting shows a moderate decline. This suggests two findings: (1) extreme cases where both probabilities are rare, making the protective weighting unnecessary; (2)  the weighting suppresses informative cases where $\log \frac{\pi_{\text{post}}}{\pi}$ carries a critical signal.
    \item Motivated by GSPO~\citep{zheng2025group}, we explore trajectory-level advantage shaping by averaging token-level disagreement in a response. But this causes substantial performance drops on out-of-domain tasks compared to our token-level design. This suggests that fine-grained, token-level shaping better preserves crucial variations in token-wise credit assignment and is more robust to distribution shifts.
    \item We remove the advantage clipping in Equation~\eqref{eq:shaped_advantage}. Performance degrades mostly in AIME2024 \& 2025, the most challenging reasoning tasks. This indicates that preserving consistent advantage directions is crucial for difficult problems, where misalignment between token-level and trajectory-level directions can severely disrupt learning.
    \item Instead of using contrastive disagreement for advantage shaping, we follow \citet{lin2025ravr} in treating it as an additional regularization term in the final objective. Results show that the use of contrastive disagreement in \ours is more effective than \citet{lin2025ravr} across most in-domain and out-of-domain tasks.

\end{itemize}

\subsection{In-depth Analysis}
\paragraph{How CPO improves RL training.}
To  understand the learning mechanisms of \ours and how it improves RL training, we analyze the training dynamics of \ours and GRPO in Figure~\ref{fig:dynamics}.
From (a)/(b), we observe that \ours achieves more effective training than GRPO on both in-domain and out-domain tasks. While the difference in greedy decoding accuracy appears modest, (c)/(d) reveal larger gaps in Pass@$K$ performance ($K=$1 to 16). GRPO primarily optimizes for pass@1 performance compared to the Base model~\cite{walder2025passk}, with diminishing gains as $K$ increases and even performance degradation on MMLU-PRO. In contrast, \ours improves both individual sample quality and diversity, with consistent improvements across all $K$ values.
The entropy and response length curves in (e)/(f) provide insight into these differences. \ours converges to a relatively high entropy plateau while producing longer responses, indicating sustained exploration throughout training. GRPO, however, drives entropy near zero, severely constraining the model's exploratory capacity.

\textbf{How response correctness shapes \ours.}
To investigate how correct and incorrect responses contribute to training \ours, we vary the ratio of combination strengths for correct and incorrect responses, i.e., $\alpha^{+}:\alpha^{-}$ in Figure~\ref{fig:cpo_ratio}. A clear pattern emerges: higher ratios of $\alpha^{+}$ yields increasing entropy throughout training. 
This entropy-increasing behavior stems from two factors. First, CPO activates all-correct groups that contribute zero gradient in GRPO, injecting additional learning signal into otherwise wasted data. Second, unlike sequence-level updates that amplify already-dominant tokens and collapse the distribution, token-level updates selectively reward correct-leaning tokens regardless of their prior probability, reinforcing diverse reasoning paths and maintaining a broader output distribution. We also provide a controlled experiment isolating CPO's contribution on zero-advantage prompts in Appendix~\ref{sec:apdx:zero}, where CPO outperforms RL-ZVP by +3.2\% average, confirming the superiority of contrastive disagreement over entropy on zero-advantage problems.
Conversely, incorrect responses focus on exploitation and confident predictions, as evidenced by steadily decreasing entropy for higher $\alpha^{-}$ ratios. 
Downstream performance in (b)/(c) verifies that balanced ratios work best, with 1:1 achieving the highest Pass@K on both AIME2025 and MMLU-PRO, outperforming exploration-heavy (5:0, 4:1) and exploitation-heavy (0:5, 1:4) configurations. 
This suggests that incorrect and correct responses are mutually beneficial: correct responses allow the model to explore novel paths, while incorrect responses take these opportunities to exploit and refine the best paths.

\begin{wrapfigure}{r}{0.5\linewidth}
    \centering
    \vspace{-5pt}
    \includegraphics[width=0.95\linewidth]{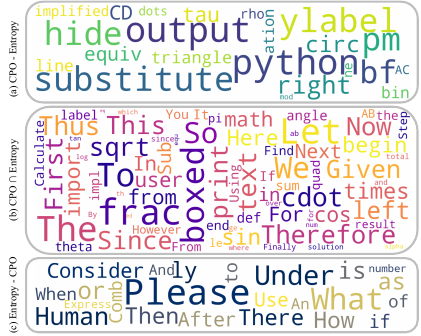}
    \caption{Word clouds of 100 tokens with top absolute disagreement (\ours) and entropy: (a) \textbf{\textit{CPO - Entropy}}: tokens only detected by \ours; (b) \textit{\textbf{CPO $\cap$ Entropy}}: tokens detected by both; and (c) \textit{\textbf{Entropy - CPO}}: tokens detected only by entropy.}
    \label{fig:word_cloud}
    \vspace{-7pt}
\end{wrapfigure}

\textbf{Contrastive disagreement \textit{vs.} entropy.}
% While the correctness-aware contrastive disagreement of \ours is not derived from entropy, the preceding analyses reveal that \ours has a significant impact on entropy. To understand this relationship and identify their key differences, we visualize word clouds of fork tokens for each metric in Figure~\ref{fig:word_cloud}.
% We observe that 76\% of the fork tokens identified by both measures overlap, which is expected since high-entropy tokens often indicate model uncertainty that correlates with potential errors. However, tokens uniquely identified by \ours are execution-oriented, including programming constructs (\texttt{python}, \texttt{output}, \texttt{label}) and precise mathematical symbols (\texttt{tau}, \texttt{equiv}, \texttt{circ}, \texttt{triangle}), reflecting its focus on discriminative features that directly impact correctness. In contrast, entropy uniquely emphasizes discourse markers (\texttt{Please}, \texttt{Consider}, \texttt{Then}) and flexible communicative expressions (\texttt{Human}, \texttt{What}, \texttt{How}), which capture linguistic variability. 
% This distinction makes \ours particularly suitable for RL training: By prioritizing tokens critical to user objectives over linguistic diversity, \ours enables more efficient RL training that directly aligns policy learning with task success.
While \ours is not derived from entropy, prior analyses show that it strongly impacts entropy. To understand their relationship and key differences, we show word clouds of fork tokens in Figure~\ref{fig:word_cloud}.
76\% of fork tokens overlap between both measures. This is expected as high-entropy tokens indicate uncertainty, which typically correlates with potential errors. However, \ours uniquely identifies execution-oriented tokens, including programming constructs (\texttt{python}, \texttt{output}) and mathematical symbols (\texttt{tau}, \texttt{equiv}), reflecting focus on correctness-critical features. Entropy emphasizes discourse markers (\texttt{Please}, \texttt{Consider}) and communicative expressions (\texttt{Human}, \texttt{What}), capturing linguistic variability.
This distinction makes \ours ideal for RL training: prioritizing task-critical tokens over linguistic diversity enables efficient RL training that closely aligns policy learning with task success.

\section{Related Work}
\textbf{Entropy-Intervened RL.}
Using entropy in RL was first introduced by~\citet{williams1991function} to encourage exploration and prevent premature convergence. Recently, several works have incorporated entropy into RLVR objectives. \citet{deng2025decomposing} analyze entropy dynamics during training, establishing a foundation for entropy-based interventions.
A straightforward approach is entropy regularization, adding entropy terms to the RLVR objective~\cite{shen2025entropy,wang2025arbitrary}. \citet{lei2025revisiting} extend this with bi-directional regularization to prevent both entropy collapse and explosion. However, these methods can suffer from instability~\cite{hao2025rethinking}, motivating approaches that focus on entropy changes rather than absolute values~\cite{hao2025rethinking,wu2025quantile}.
Another line of work uses entropy for advantage shaping, including awarding high entropy for exploration~\cite{cheng2025reasoning}, suppressing entropy to consolidate well-performing patterns~\cite{agarwal2025unreasonable}, and hybrid approaches~\cite{wang2025harnessing,le2025no}.
Beyond regularization and advantage shaping, entropy enables additional control mechanisms. \citet{wang2025beyond} and \citet{zheng2025first} use entropy to identify critical fork tokens or chunks to narrow training focus.
\citet{sheng2025espo} employ it to control trust regions during updates.

\textbf{Fine-grained Correctness Signals.}
Recent work explores fine-grained rewards to address sparse reward problem in RLVR~\cite{li-etal-2025-red}. \citet{lightman2024lets,wu2023fine} provide step-level correctness but require intensive human labor. Tree-based search methods automate step-level annotation~\cite{yao2023tree,chia-etal-2024-reasoning,wang-etal-2024-math,jiao-etal-2024-learning,li-etal-2025-elicit}, but the repeated sampling makes them computationally impractical in RL settings.
Recent advances in on-policy distillation (OPD) enable students to learn from teacher 
feedback on self-generated outputs~\cite{agarwal2024onpolicy}, enabling more effective learning.
To eliminate external teacher dependency, self-distillation methods have emerged. Rejection sampling uses output rewards to select superior responses~\cite{touvron2023llama,yang-etal-2024-self}, while others simulate synthetic teachers by conditioning the model on critics~\cite{xu-etal-2024-reasons,hubotter2026reinforcement}, reference~\cite{lin2025ravr,shenfeld2026self}  or any other correctness-informed information~\cite{wang2026openclaw}.

\section{Conclusion}
We propose Contrastive Policy Optimization (\ours), a novel framework that leverages token-level contrastive disagreement between reference-guided and vanilla generation as a correctness-aware signal for advantage shaping in RLVR. Both theoretical and empirical results demonstrate that this contrastive disagreement reliably indicates token-level correctness. \ours achieves significant improvements on challenging math reasoning tasks while maintaining out-of-domain generalization compared to entropy-based methods.
Our analysis reveals that \ours outperforms GRPO by better preserving the exploration capability of the backbone model. This advantage stems from the distinct roles of different response types: correct responses drive exploration of novel solutions, while incorrect responses exploit the most promising paths. Balance both leads to superior performance.
\bibliography{example_paper}
\bibliographystyle{icml2026}

%%%%%%%%%%%%%%%%%%%%%%%%%%%%%%%%%%%%%%%%%%%%%%%%%%%%%%%%%%%%%%%%%%%%%%%%%%%%%%%
%%%%%%%%%%%%%%%%%%%%%%%%%%%%%%%%%%%%%%%%%%%%%%%%%%%%%%%%%%%%%%%%%%%%%%%%%%%%%%%
% APPENDIX
%%%%%%%%%%%%%%%%%%%%%%%%%%%%%%%%%%%%%%%%%%%%%%%%%%%%%%%%%%%%%%%%%%%%%%%%%%%%%%%
%%%%%%%%%%%%%%%%%%%%%%%%%%%%%%%%%%%%%%%%%%%%%%%%%%%%%%%%%%%%%%%%%%%%%%%%%%%%%%%
\newpage
\appendix
\onecolumn
\section{Limitations}
\label{sec:apdx:limitations}
The primary computational overhead of \ours stems from the additional reference-guided 
forward pass required to compute the posterior likelihood $\pi_{\mathrm{post}}$, which 
increases training time by approximately 20\% compared to standard GRPO. One promising 
direction to mitigate this is to decouple the reference-guided forward pass from the 
main training loop via asynchronous computation, at the cost of slightly stale posterior estimates that introduce a degree of off-policy bias.

Additionally, \ours currently relies on the availability of reference answers to 
construct the correctness-informed distribution $\pi_{\mathrm{post}}$, which limits its applicability to settings with verifiable rewards. However, as discussed in 
Section~\ref{sec:method}, reference answers are just one instantiation of a 
correctness-informed distribution. A natural extension is to replace reference-guided 
prompting with other contextual signals that can serve the same role without requiring 
ground-truth answers, such as critic feedback that judges the quality of the model's own outputs, or process-level annotations that identify reasoning errors at intermediate steps. Exploring such alternatives would broaden the applicability of \ours to open-ended generation and subjective reasoning tasks, where reference answers are unavailable.

\section{Impact Statements}
\label{sec:apdx:impact}
The primary goal of this work is to improve the accuracy of credit assignment during reinforcement learning. This method has the potential for social impact in domains requiring high accuracy, such as industrial workflow agent and scientific assistance. While improving reasoning capabilities is generally beneficial, we acknowledge that advanced models carry inherent risks related to dual-use potential and the amplification of biases present in training data. We do not foresee specific negative consequences unique to this method beyond these broader considerations that apply to general-purpose LLM development.

\section{Theoretical Analysis of Contrastive Generation Disagreement}
\label{apdx:sec:theory}

Let $x$ denote a prompt/question and $y=(y_1,\dots,y_T)$ denote an autoregressive rollout from the \emph{prior} (vanilla) policy
\begin{equation}
\pi(y\mid x)\;=\;\prod_{t=1}^T \pi_{\theta}(y_t\mid x,y_{<t}).
\end{equation}
A verifiable environment returns a binary reward
\begin{equation}
R(x,y)\in\{-1,1\},
\end{equation}
indicating whether the final answer in $y$ is correct. GRPO optimizes the standard RLVR objective
\begin{equation}
J(\theta)\;=\;\mathbb{E}_{x}\,\mathbb{E}_{y\sim \pi(\cdot\mid x)}\big[ R(x,y)\big].
\end{equation}

\begin{definition}[Oracle-defined correctness event]
\label{def:correctness_event}
For each $x$, assume there exists oracle ground-truth information (e.g., a reference answer/solution) that uniquely determines correctness for any candidate completion $y$.
Let
\begin{equation}
C \;:=\; \{ R(x,Y)=1 \}
\end{equation}
denote the event that a random completion $Y$ is correct \emph{with respect to the oracle}. Note that $C$ is oracle-defined and model-agnostic; the model only induces a distribution over $Y$.
\end{definition}

\begin{definition}[Position-wise oracle correctness probability]
\label{def:q}
For any $x$, prefix $y_{<t}$, and candidate token value $y_t$, define
\begin{equation}
g(x,y_{<t},y_t)
\;:=\;
\mathbb{P}\!\left( C \,\middle|\, X=x,\; Y_{<t}=y_{<t},\; Y_t=y_t \right).
\end{equation}
Equivalently, $g(x,y_{<t},y_t)$ is the probability (over all possible future continuations after choosing $y_t$ at position $t$) that the final completion is correct.
\end{definition}

\begin{definition}[Reference-guided posterior scoring and disagreement]
\label{def:delta}
Let $y^\star(x)$ denote the reference ground-truth,
we define a \emph{reference-guided} policy by conditioning the same model on $x\oplus y^\star(x)$:
\begin{equation}
\pi_{\mathrm{post}}(y\mid x)
\;=\;
\prod_{t=1}^T \pi_{\theta}(y_t\mid x,y^\star(x),y_{<t}).
\end{equation}
Given a rollout $y\sim \pi (\cdot\mid x)$ (which may be correct or incorrect), CPO re-scores the same tokens under $\pi_{\mathrm{post}}$ and defines the token-wise \emph{contrastive generation disagreement}
\begin{equation}
\delta_t(x,y)
\;:=\;
\log \pi_{\mathrm{post}}(y_t\mid x,y^\star(x),y_{<t})
\;-\;
\log \pi(y_t\mid x,y_{<t}).
\end{equation}
\end{definition}

\begin{definition}[Ideal correctness-conditioned token distribution]
\label{def:tildepi}
Define the \emph{ideal} next-token distribution of the prior policy conditioned on oracle correctness:
\begin{equation}
\tilde{\pi}_{\mathrm{post}}(y_t\mid x,y_{<t})
\;:=\;
\mathbb{P}\!\left( Y_t=y_t \,\middle|\, X=x,\; Y_{<t}=y_{<t},\; C \right).
\end{equation}
This distribution is uniquely determined by $\pi$ and the oracle-defined event $C$.
\end{definition}

\begin{lemma}[Bayes reweighting under oracle correctness]
\label{lem:bayes_reweighting}
For any $x$, prefix $y_{<t}$, and token $y_t$,
\begin{equation}
\tilde{\pi}_{\mathrm{post}}(y_t\mid x,y_{<t})
\;=\;
\frac{
\pi(y_t\mid x,y_{<t})\; g(x,y_{<t},y_t)
}{
Z(x,y_{<t})
},
\;\;\;
Z(x,y_{<t})
\;:=\;
\sum_{b}
\pi(b\mid x,y_{<t})\; g(x,y_{<t},b).
\end{equation}
Equivalently, $Z(x,y_{<t})=\mathbb{P}(C\mid X=x, Y_{<t}=y_{<t})$ is the prior-average correctness probability at that prefix.
\end{lemma}

\begin{proof}
By Definition~\ref{def:tildepi} and Bayes' rule,
\begin{align}
\tilde{\pi}_{\mathrm{post}}(y_t\mid x,y_{<t})
&=
\frac{
\mathbb{P}(Y_t=y_t, C \mid X=x, Y_{<t}=y_{<t})
}{
\mathbb{P}(C \mid X=x, Y_{<t}=y_{<t})
}.
\end{align}
Factor the numerator:
\begin{align}
\mathbb{P}(Y_t=y_t, C \mid x, y_{<t})
&=
\mathbb{P}(Y_t=y_t \mid x, y_{<t})\;
\mathbb{P}(C \mid x, y_{<t}, Y_t=y_t) \\
&=
\pi(y_t\mid x,y_{<t})\; g(x,y_{<t},y_t),
\end{align}
where the last equality uses Definition~\ref{def:q}. The denominator equals
\begin{align}
\mathbb{P}(C \mid x, y_{<t})
=
\sum_b
\mathbb{P}(Y_t=b \mid x, y_{<t})\;
\mathbb{P}(C \mid x, y_{<t}, Y_t=b)
=
\sum_b
\pi(b\mid x,y_{<t})\; g(x,y_{<t},b),
\end{align}
which defines $Z(x,y_{<t})$. Substituting into the Bayes expression yields the claim.
\end{proof}

\begin{corollary}[Ideal log-ratio equals log-correctness up to a prefix constant]
\label{cor:ideal_logratio}
For any $x,y_{<t},y_t$ with $\pi(y_t\mid x,y_{<t})>0$ and $g(x,y_{<t},y_t)>0$,
\begin{equation}
\log \tilde{\pi}_{\mathrm{post}}(y_t\mid x,y_{<t})
-
\log \pi(y_t\mid x,y_{<t})
=
\log g(x,y_{<t},y_t)
-
\log Z(x,y_{<t}).
\end{equation}
\end{corollary}

\begin{proof}
Take logarithms on both sides of Lemma~\ref{lem:bayes_reweighting} and subtract $\log \pi(y_t\mid x,y_{<t})$.
\end{proof}

\begin{assumption}[Reference answer encodes correctness information]
\label{assump:hint_proxy}
In practice, we do not have direct access to the ideal distribution $\tilde{\pi}_{\mathrm{post}}$. 
However, we can condition the model on the oracle reference $y^\star$ (e.g.,  prompting the model to refine based on $y^\star(x)$) to obtain $\pi_{\mathrm{post}}(y_t\mid x,y^\star(x),y_{<t})$.
Since $y^\star$ encodes the ground-truth that determines correctness and the prompt also asks the model to refine, conditioning on such prompt serves as a practical proxy for conditioning on the correctness event $C$,  i.e., $\pi_{\mathrm{post}}(\cdot\mid x,y^\star(x),y_{<t}) \approx \tilde{\pi}_{\mathrm{post}}(\cdot\mid x,y_{<t})$.
% Let $\mathcal{O}(x)$ denote oracle information that fully determines correctness for question $x$ (e.g., a reference solution/answer).
% Assume the hint is an encoding of the oracle, $h(x)=\phi(\mathcal{O}(x))$, and that conditioning the model on $h(x)$ primarily injects this oracle correctness information rather than unrelated stylistic bias.
% Consequently, the hint-guided token distribution $\pi_{\mathrm{post}}(\cdot\mid x,h(x),y_{<t})$ serves as a practical proxy for the ideal correctness-conditioned distribution $\tilde{\pi}_{\mathrm{post}}(\cdot\mid x,y_{<t})$.
\end{assumption}

\paragraph{Justification (empirical).}
We implement this proxy via teacher-forcing: for a prior rollout $y\sim \pi(\cdot\mid x)$, we re-compute $\log \pi_{\mathrm{post}}(y_t\mid x,y^\star(x),y_{<t})$ on the same tokens. Empirically, in Sec.~\ref{sec:emp_study}, large negative disagreements are concentrated on incorrect trajectories and align with error-causing steps, supporting $\pi_{\mathrm{post}}(\cdot\mid x,y^\star(x),y_{<t})$ as an effective approximation to $\tilde{\pi}_{\mathrm{post}}(\cdot\mid x,y_{<t})$.

\begin{theorem}[Contrastive generation disagreement indicates incorrectness]
\label{thm:disagreement_incorrectness}
Fix $x$ and a prefix $y_{<t}$. Define the \emph{ideal} disagreement for token $y_t$ as
\begin{equation}
\tilde{\delta}_t(x,y_{<t},y_t)
\;:=\;
\log \tilde{\pi}_{\mathrm{post}}(y_t\mid x,y_{<t})
-
\log \pi(y_t\mid x,y_{<t}).
\end{equation}
Then, for any $y_t$ with $\pi(y_t\mid x,y_{<t})>0$ and $g(x,y_{<t},y_t)>0$,
\begin{equation}
\tilde{\delta}_t(x,y_{<t},y_t)
=
\log g(x,y_{<t},y_t)
-
\log Z(x,y_{<t}),
\end{equation}
where $Z(x,y_{<t})$ is given in Lemma~\ref{lem:bayes_reweighting}. In particular, for fixed $(x,y_{<t})$, $\tilde{\delta}_t(x,y_{<t},y_t)$ is a strictly increasing function of $g(x,y_{<t},y_t)$, and
\begin{equation}
\begin{split}
\tilde{\delta}_t(x,y_{<t},y_t) < 0
\quad\Longleftrightarrow\quad
g(x,y_{<t},y_t) < Z(x,y_{<t}),\\
\tilde{\delta}_t(x,y_{<t},y_t) > 0
\quad\Longleftrightarrow\quad
g(x,y_{<t},y_t) > Z(x,y_{<t}).
\end{split}
\end{equation}
Therefore, negative ideal disagreement identifies \emph{incorrect-leaning} tokens whose oracle correctness probability is below the prior-average correctness at that prefix.
\end{theorem}

\begin{proof}
The identity follows immediately from Corollary~\ref{cor:ideal_logratio}.
The sign equivalences follow from exponentiating
$\tilde{\delta}_t(x,y_{<t},y_t)=\log\!\big(g(x,y_{<t},y_t)/Z(x,y_{<t})\big)$.
\end{proof}

\begin{corollary}[Interpretation for CPO scoring on any rollout]
\label{cor:cpo_interpretation}
Let $y\sim \pi(\cdot\mid x)$ be any rollout, which may satisfy $R(x,y)=1$ or $R(x,y)=-1$.
At each position $t$, CPO computes
\[
\delta_t(x,y)
=
\log \pi_{\mathrm{post}}(y_t\mid x,y^\star(x),y_{<t})
-
\log \pi(y_t\mid x,y_{<t}).
\]
Under Assumption~\ref{assump:hint_proxy}, $\delta_t(x,y)$ serves as a practical proxy for the ideal quantity
$\tilde{\delta}_t(x,y_{<t},y_t)$, and thus inherits its correctness-aligned interpretation:
large negative $\delta_t(x,y)$ indicates that the chosen token $y_t$ is suppressed by oracle-informed guidance and is therefore likely to be \emph{incorrect-leaning} at that prefix, which is also empirically validated in Sec.~\ref{sec:emp_study}.
\end{corollary}

\section{Generation Prompts}
\label{apdx:sec:prompt}
In this section, we provide the complete prompts used in our experiments: vanilla generation (Table~\ref{tab:prompt_vanilla}), reference-guided generation (Table~\ref{tab:prompt_rgg}), and two variants from the ablation study in Sec.~\ref{sec:ablation}: multi-turn gold-shot (Table~\ref{tab:prompt_mtgs}) and multi-turn 1-shot generation (Table~\ref{tab:prompt_mt1s}).
While the exact prompt format is not the core contribution of \ours, it plays a practical role in determining how effectively the reference-guided distribution shifts toward correctness-oriented reasoning. As long as the prompt elicits a posterior distribution that is more correctness-informed than the vanilla generation, the contrastive disagreement remains a meaningful signal regardless of the specific format. The performance differences across prompt variants in Table~\ref{tab:ablation} reflect how effectively each format elicits correctness-oriented reasoning, rather than a fundamental sensitivity to prompt design. 
Based on our ablation study, we suggest the following practical guidelines for reference-guided prompt design: (1) use a single-turn format, as LLMs have stronger single-turn instruction-following capability; (2) generate in a first-person, real-time problem-solving tone to encourage genuine reasoning; (3) instruct the model not to acknowledge the reference answer, ensuring the model reasons rather than copies; (4) require complete reasoning with intermediate steps and error-correction to capture meaningful token-level differences.

\begin{table*}[h!]
\centering
\caption{Prompt used in vanilla generation.}
\label{tab:prompt_vanilla}
\begin{tcolorbox}[colframe=blue!50, colback=blue!5, fonttitle=\bfseries\large, coltitle=black, boxrule=0.5mm, arc=5mm, auto outer arc, width=\textwidth,toptitle=6pt, bottomtitle=6pt]
\small
\setstretch{1.2} %
\texttt{<|im\_start|>system\\
You are a helpful assistant.<|im\_end|>\\
<|im\_start|>user\\
{\textbf{[question]}}\\
Please reason step by step, and put your final answer within \textbackslash boxed\{\}.<|im\_end|>\\
<|im\_start|>assistant}
\end{tcolorbox}
\end{table*}

\begin{table*}[h!]
\centering
\caption{Prompt used in reference-guided generation.}
\label{tab:prompt_rgg}
\begin{tcolorbox}[colframe=blue!50, colback=blue!5, fonttitle=\bfseries\large, coltitle=black, boxrule=0.5mm, arc=5mm, auto outer arc, width=\textwidth,toptitle=6pt, bottomtitle=6pt]
\small
\setstretch{1.2} %
\texttt{<|im\_start|>system\\
You are a helpful assistant.<|im\_end|>\\
<|im\_start|>user\\
\#\#\#Question: \textbf{[question]}\\
\#\#\#Reference Answer: \textbf{[reference answer]}\\
Given a question and reference answer, generate a first-person step-by-step reasoning process leading to the answer in \textbackslash boxed\{\}.\\
\newline
Requirements:\\
1. Output only the first-person thought process, no preamble or summary\\
2. Simulate real-time problem-solving tone, avoid meta-commentary\\
3. Don't mention or imply knowing the reference answer (avoid 'according to the answer...' etc.)\\
4. Show complete reasoning path with intermediate steps, verification, and error-correction, not just restatement\\
5. Put final answer in \textbackslash boxed\{\}<|im\_end|>\\
<|im\_start|>assistant}
\end{tcolorbox}
\end{table*}

\begin{table*}[h!]
\centering
\caption{Prompt used in multi-turn gold-shot generation.}
\label{tab:prompt_mtgs}
\begin{tcolorbox}[colframe=blue!50, colback=blue!5, fonttitle=\bfseries\large, coltitle=black, boxrule=0.5mm, arc=5mm, auto outer arc, width=\textwidth,toptitle=6pt, bottomtitle=6pt]
\small
\setstretch{1.2} %
\texttt{<|im\_start|>system\\
You are a helpful assistant.<|im\_end|>\\
<|im\_start|>user\\
{\textbf{[question]}}<|im\_end|>\\
<|im\_start|>assistant\\
\textbf{[reference answer]}<|im\_end|>\\
<|im\_start|>user\\
{\textbf{[question]}}\\
Please reason step by step, and put your final answer within \textbackslash boxed\{\}.<|im\_end|>\\
<|im\_start|>assistant}
\end{tcolorbox}
\end{table*}

\begin{table*}[h!]
\centering
\caption{Prompt used in multi-turn 1-shot generation.}
\label{tab:prompt_mt1s}
\begin{tcolorbox}[colframe=blue!50, colback=blue!5, fonttitle=\bfseries\large, coltitle=black, boxrule=0.5mm, arc=5mm, auto outer arc, width=\textwidth,toptitle=6pt, bottomtitle=6pt]
\small
\setstretch{1.2} %
\texttt{<|im\_start|>system\\
You are a helpful assistant.<|im\_end|>\\
<|im\_start|>user\\
Let $S = \{5^k | k \in \mathbb{Z}, 0 \le k \le 2004 \}$. Given that $5^{2004} = 5443 \cdots 0625$ has $1401$ digits, how many elements of $S$ begin with the digit $1$?<|im\_end|>\\
<|im\_start|>assistant\\
Note that $5^n$ has the same number of digits as $5^{n-1}$ if and only if $5^{n-1}$ has a leading digit $1$. Therefore, there are 2004 - 1401 = 603 numbers with leading digit $1$ among the set $\{5^1, 5^2, 5^3, \cdots 5^{2003}\}.$ However, $5^0$ also starts with $1$, so the answer is 603 + 1 = \textbackslash boxed\{604\}.<|im\_end|>\\
<|im\_start|>user\\
{\textbf{[question]}}\\
Please reason step by step, and put your final answer within \textbackslash boxed\{\}.<|im\_end|>\\
<|im\_start|>assistant}
\end{tcolorbox}
\end{table*}

\section{Detailed training Implementations}
\label{apdx:sec:impl}
To verify the effectiveness of \ours in enhancing RLVR performance, we conduct experiments on two base models: the general-purpose Qwen3-Base-4B~\cite{yang2025qwen3} and the domain-specific Qwen2.5-Math-7B~\cite{yang2024qwen2}. We train both models on the MATH dataset~\cite{hendrycks2021measuring}, which contains 7.5k problems. We implement \ours based on GRPO~\cite{shao2024deepseekmath}.
During training, we use a prompt batch size of 1024 with 8 rollouts per prompt, and set the sampling temperature to 1.0. The maximum context length is set to 4096 for Qwen2.5-Math-7B and 8192 for Qwen3-Base-4B. The combination strength $\alpha^{+}$ and $\alpha^{-}$ are both set to 0.025. We perform policy updates with a mini-batch size of 256 and a learning rate of 1e-6. The KL penalty coefficient between the policy and reference models is set to 1e-3, and the PPO clip ratio is set to 0.2. We conduct all experiments on 8 A800 GPUs. The additional forward pass required to calculate reference-guided generation probabilities increases training time by approximately 20\%.
\section{Additional Experiments and Analyses}

\begin{table}[h]
\centering
\caption{Mean $\pm$ std over 3 independent runs on Qwen2.5-Math-7B.}
\label{tab:std}
\resizebox{\textwidth}{!}{
\begin{tabular}{lcccccccc}
\toprule
\textbf{Methods} & \textbf{MATH-500} & \textbf{AIME2024} & \textbf{AIME2025} & \textbf{AMC2023} & \textbf{GPQA} & \textbf{MMLU-PRO} & \textbf{KnowLogic} & \textbf{Avg.} \\
\midrule
GRPO & 88.0\scriptsize{±0.6} & 45.5\scriptsize{±3.9} & 30.0\scriptsize{±3.3} & 85.0\scriptsize{±2.5} & 61.2\scriptsize{±3.1} & 61.8\scriptsize{±1.3} & 65.0\scriptsize{±3.1} & 62.4\scriptsize{±0.7} \\
DAPO & 87.8\scriptsize{±1.4} & 38.9\scriptsize{±10.3} & 30.0\scriptsize{±3.3} & 89.2\scriptsize{±3.8} & 62.3\scriptsize{±4.1} & 61.2\scriptsize{±4.0} & 64.0\scriptsize{±3.8} & 61.9\scriptsize{±1.1} \\
CPO  & 88.8\scriptsize{±0.4} & 58.9\scriptsize{±1.9} & 40.0\scriptsize{±3.3} & 92.5\scriptsize{±2.5} & 68.0\scriptsize{±2.3} & 68.4\scriptsize{±1.5} & 68.3\scriptsize{±0.5} & 69.4\scriptsize{±0.8} \\
\bottomrule
\end{tabular}}
\end{table}

\subsection{Statistical Significance.}
\label{sec:apdx:statis}
To verify the reliability of our results, we conduct 3 independent runs for CPO, GRPO, 
and DAPO and report mean $\pm$ std in Table~\ref{tab:std}.
CPO consistently outperforms both baselines with improvements well beyond one standard 
deviation across all benchmarks. Notably, CPO exhibits substantially smaller variance 
than DAPO on challenging benchmarks such as AIME2024 (±1.9 vs. ±10.3), suggesting more 
stable training dynamics. We also note that our primary evaluation metric Pass@16 
inherently mitigates evaluation variance by aggregating over 16 sampled responses per 
problem, further reducing the influence of random sampling.

\subsection{CPO on the Zero-Advantage Problem}
\label{sec:apdx:zero}
\begin{table}[h]
\centering
\caption{Performance comparison of CPO, RL-ZVP, and the base model across benchmarks under the zero-advantage controlled setting.}
\label{tab:zero_advantage}
\adjustbox{max width=\textwidth}{
\setlength{\tabcolsep}{1mm}{
\begin{tabular}{lccccccccc}
\toprule
\textbf{Methods} & \textbf{MATH-500} & \textbf{AIME2024} & \textbf{AIME2025} & \textbf{AMC2023} & \textbf{GPQA} & \textbf{MMLU-PRO} & \textbf{KnowLogic} & \textbf{Avg.} \\
\midrule
Qwen2.5-Math-7B & 87.8 & 43.3 & 30.0 & 85.0 & \textbf{70.5} & 63.4 & 55.8 & 62.3 \\
+ RL-ZVP          & 87.2 & 43.3 & 30.0 & 85.0 & 67.1 & 63.1 & 53.2 & 61.3 \\
+ CPO             & \textbf{88.8} & \textbf{50.0} & \textbf{33.3} & \textbf{87.5} & 69.6 & \textbf{64.3} & \textbf{58.0} & \textbf{64.5} \\
\bottomrule
\end{tabular}}}
\end{table}
For zero-advantage prompts, existing entropy-based methods such as RL-ZVP attempt to assign a high advantage for high-entropy tokens. However, since entropy only captures uncertainty without distinguishing whether a token contributes to correct or incorrect reasoning, the resulting signal remains correctness-agnostic: it may inadvertently reinforce confused or erroneous token choices that happen to exhibit high entropy. CPO addresses this in a more principled manner. CPO uses the token-level contrastive disagreement $\delta_{i,t}$ as the shaped advantage, which is theoretically and empirically demonstrated in Section 3.1 to serve as a reliable proxy for token-level correctness probability. This provides a more meaningful and directionally reliable advantage signal compared to existing entropy-intervened methods. In our experiments, this advantage is reflected in the consistently superior performance of CPO over RL-ZVP across both in-domain and out-of-domain benchmarks (\textit{e.g.}, +9.5\% average on Qwen2.5-Math-7B and +6.6\% on Qwen3-Base-4B). To further isolate the contribution of CPO on zero-advantage prompts, we conduct a controlled experiment where only rollouts from zero-advantage groups are used for training, comparing advantage shaping of CPO and RL-ZVP under this setting. The results demonstrate that CPO consistently outperforms RL-ZVP under this controlled zero-advantage setting, with a +3.2\% average improvement. This confirms that the advantage signals of CPO is more useful over entropy-based methods.

\subsection{Comparison with SFT and Distillation Methods}
\label{sec:apdx:sft}
\begin{table}[h]
\centering
\caption{Performance comparison of CPO with SFT and distillation baselines on Qwen2.5-Math-7B.}
\label{tab:sft_distill_comparison}
\adjustbox{max width=\textwidth}{
\setlength{\tabcolsep}{1mm}{
\begin{tabular}{lccccccccc}
\toprule
\textbf{Methods} & \textbf{MATH-500} & \textbf{AIME2024} & \textbf{AIME2025} & \textbf{AMC2023} & \textbf{GPQA} & \textbf{MMLU-PRO} & \textbf{KnowLogic} & \textbf{Avg.} \\
\midrule
Original    & 87.8 & 43.3 & 30.0 & 85.0 & 70.5 & 63.4 & 55.8 & 62.3 \\
\midrule
SFT         & 79.6 & 23.3 & 16.6 & 65.0 & 52.0 & 49.1 & 61.4 & 49.6 \\
Distillation & 85.2 & 43.3 & 33.3 & 85.0 & 77.9 & 65.7 & 70.3 & 65.8 \\
\textbf{CPO} & \textbf{89.0} & \textbf{60.0} & \textbf{43.3} & \textbf{92.5} & \textbf{68.8} & \textbf{69.0} & \textbf{68.9} & \textbf{70.2} \\
\bottomrule
\end{tabular}}}
\end{table}

To ensure fair comparison under equivalent data conditions, we conduct additional experiments comparing CPO with SFT and distillation baselines on Qwen2.5-Math-7B, where the distillation baseline is constructed by fine-tuning the model on trajectories self-sampled with reference-guided prompts.

As shown in Table~\ref{tab:sft_distill_comparison}, vanilla SFT degrades substantially (-12.7 avg.) compared to the original Qwen2.5-Math-7B, while distillation recovers much of the performance (+3.5 avg.), and CPO achieves the best results by a significant margin (+7.9 avg.).
SFT suffers from a fundamental distribution mismatch: the concise ground-truth trajectories are far from the model's own generation distribution, making them difficult to learn from directly. Distillation mitigates this by using reference-guided self-sampling to generate trajectories that are initially much closer to the model's own distribution. Moreover, the additional reference information guides the model to produce more reasonable and learnable samples, which explains the consistent performance gains over vanilla SFT.
However, distillation remains inherently off-policy: the sampled trajectories are collected once and fixed as static supervised targets. As training progresses and the policy evolves, these fixed targets become increasingly misaligned with the current model distribution, ultimately limiting its performance. CPO avoids this issue entirely by operating fully on-policy: the reference answer is never used as a imitation target, but solely to compute a token-level contrastive signal at each training step using the current policy, allowing  the contrastive signal to evolve alongside the model and preserving full exploration capability.

% You can have as much text here as you want. The main body must be at most $8$
% pages long. For the final version, one more page can be added. If you want, you
% can use an appendix like this one.

% The $\mathtt{\backslash onecolumn}$ command above can be kept in place if you
% prefer a one-column appendix, or can be removed if you prefer a two-column
% appendix.  Apart from this possible change, the style (font size, spacing,
% margins, page numbering, etc.) should be kept the same as the main body.
%%%%%%%%%%%%%%%%%%%%%%%%%%%%%%%%%%%%%%%%%%%%%%%%%%%%%%%%%%%%%%%%%%%%%%%%%%%%%%%
%%%%%%%%%%%%%%%%%%%%%%%%%%%%%%%%%%%%%%%%%%%%%%%%%%%%%%%%%%%%%%%%%%%%%%%%%%%%%%%

\end{document}